\def\year{2022}\relax
\documentclass[letterpaper]{article} 
\usepackage{aaai22}  
\usepackage{times}  
\usepackage{helvet}  
\usepackage{courier}  
\usepackage[hyphens]{url}  
\usepackage{graphicx} 
\usepackage{natbib}  
\usepackage{caption} 
\DeclareCaptionStyle{ruled}{labelfont=normalfont,labelsep=colon,strut=off} 
\usepackage{algorithm}
\usepackage{algorithmic}

\usepackage{amsmath}\allowdisplaybreaks
\usepackage{amsfonts,bm}
\usepackage{amssymb}
\usepackage{dsfont}
\usepackage{amsthm}
\usepackage{algorithm}
\usepackage{algorithmic}

\theoremstyle{plain}
\newtheorem{theorem}{Theorem}
\newtheorem{definition}{Definition}

\newtheorem{lemma}{Lemma}

\def\x{{\bf x}}

\def\tPsi{{\tilde{\Psi}}}
\def\tPhi{{\tilde{\Phi}}}

\def\hL{{\hat{L}}}
\def\tL{{\tilde{L}}}
\def\hl{{\hat{\ell}}}

\def\vone{{\bf{1}}}

\DeclareMathOperator{\KL}{KL}

\def\cD{{\mathcal{D}}}

\def\cH{{\mathcal{H}}}
\def\cI{{\mathcal{I}}}

\def\cS{{\mathcal{S}}}

\def\cX{{\mathcal{X}}}

\def\BE{{\mathbb{E}}}

\def\BP{{\mathbb{P}}}

\def\BR{{\mathbb{R}}}

\def\xtij{x_{t,i,j}}
\def\ltij{\hl_{t,i,j}}

\newcommand{\Conv}{\mathrm{Conv}}

\DeclareMathOperator*{\argmax}{arg\,max}
\DeclareMathOperator*{\argmin}{arg\,min}

\allowdisplaybreaks
\usepackage{tabularx}
\usepackage{booktabs}
\usepackage{multirow}
\usepackage{graphicx,subcaption,caption}
\usepackage{colortbl,xcolor}

\usepackage{hyperref}
\hypersetup{
  colorlinks   = true, 
  urlcolor     = blue, 
  linkcolor    = blue, 
  citecolor   = blue 
}

\newcommand{\inner}[2]{\langle#1, #2\rangle}

\newcommand{\bOne}[1]{\mathds{1}\left\{#1\right\}}
\newcommand{\EE}[2]{\mathbb{E}_{#1} \left[#2\right]}

\newcommand{\Ber}{\mathrm{Ber}}
\mathchardef\mhyphen="2D
\newcommand{\ftrlpbm}{{\tt FTRL\mhyphen PBM}}
\usepackage{color}

\usepackage{newfloat}
\usepackage{listings}
\floatstyle{ruled}
\newfloat{listing}{tb}{lst}{}
\floatname{listing}{Listing}
\title{Simultaneously Learning Stochastic and Adversarial Bandits under the Position-Based Model}
\author{
    Cheng Chen\textsuperscript{\rm 1},
    Canzhe Zhao\textsuperscript{\rm 2},
    Shuai Li \textsuperscript{\rm 2}\thanks{Corresponding author}
}
\affiliations{
    \textsuperscript{\rm 1} Nanyang Technological University\\
    \textsuperscript{\rm 2} Shanghai Jiao Tong University\\
    cheng.chen@ntu.edu.sg, canzhezhao@sjtu.edu.cn, shuaili8@sjtu.edu.cn

}

\begin{document}

\maketitle
\setcounter{secnumdepth}{0}

\begin{abstract}
Online learning to rank (OLTR) interactively learns to choose lists of items from a large collection based on certain click models that describe users' click behaviors. Most recent works for this problem focus on the stochastic environment where the item attractiveness is assumed to be invariant during the learning process. In many real-world scenarios, however, the environment could be dynamic or even arbitrarily changing. This work studies the OLTR problem in both stochastic and adversarial environments under the position-based model (PBM). We propose a method based on the follow-the-regularized-leader (FTRL) framework with Tsallis entropy and develop a new self-bounding constraint especially designed for PBM. We prove the proposed algorithm simultaneously achieves $O(\log{T})$ regret in the stochastic environment and $O(m\sqrt{nT})$ regret in the adversarial environment, where $T$ is the number of rounds, $n$ is the number of items and $m$ is the number of positions. We also provide a lower bound of order $\Omega(m\sqrt{nT})$ for adversarial PBM, which matches our upper bound and improves over the state-of-the-art lower bound. 
The experiments show that our algorithm could simultaneously learn in both stochastic and adversarial environments and is competitive compared to existing methods that are designed for a single environment.
\end{abstract}

\section{Introduction}
Learning to rank is widely used in online web search and recommender systems which selects a small group of items to present in a limited number of positions after a user starts a search session \cite{liu2011learning}. Online learning to rank (OLTR) is to learn the best ranking policy through user interactions and aims to maximize user satisfaction, \emph{e.g.} the number of user clicks, during the learning period. To understand the click signals received from users on given ranked lists of items, many click models are introduced and studied \cite{chuklin2015click}. One of the most popular click models adopted in the industry is the position-based model (PBM) \cite{richardson2007predicting} due to its simplicity and effectiveness to characterize the click rate as a product of item attractiveness and position bias. PBM is studied in OLTR setting with theoretical analysis on regret \cite{lagree2016multiple,komiyama2017position}, which is in expectation the difference of the received clicks from the clicks of the best policy. Some other works in OLTR study the cascade model \cite{kveton2015cascading,li2016contextual,zong2016cascading} and general click model \cite{zoghi2017online,lattimore2018toprank,li2019online}.

Most existing works in OLTR focus on the stochastic environment where the item attractiveness and position examination probabilities, if any, are assumed to be fixed through the learning process. However, this usually is a strong assumption in real applications where the item attractiveness could change dynamically, like the clothes interest of users might periodically change across seasons. The algorithms designed in the stochastic environment might fail to converge if the stochastic assumptions do not hold. This motivates the study of adversarial environment where the involved samples are arbitrary on a bounded domain. Usually the regret guarantee of algorithms designed under adversarial environment can only be $O(\sqrt{T})$, even in the stochastic environment whose best algorithm can achieve a much better regret of $O(\log(T))$. It is an interesting topic in online learning if there is an algorithm that can achieve $O(\sqrt{T})$ regret if run in the adversarial environment and $O(\log(T))$ regret if run in the stochastic environment. This problem is also called best-of-both-worlds (BOBW). Some works study this problem in classical multi-armed bandit problem (MAB) \cite{bubeck2012best,seldin2014one,auer2016algorithm,seldin2017improved,wei2018more,zimmert2019optimal} and combinatorial MAB (CMAB) with semi-bandit feedback \cite{zimmert2019beating}. It is an open question if we can design BOBW algorithms in OLTR whose adversarial formulation needs to be well deliberated first. In this work we hope to answer this question under the commonly adopted PBM.

We propose an algorithm for OLTR under PBM based on the follow-the-regularized-leader (FTRL) framework to simultaneously learn in stochastic and adversarial environments. Though OLTR under PBM can be formulated as a special case of CMAB by regarding the pair of an item and a position as a base arm, the direct application of existing studies does not hold. One of the main challenges is that the commonly defined suboptimality gap could be negative in the PBM setting, making it impossible to follow the existing \textit{self-bounding} technique in \cite{zimmert2019beating,zimmert2019optimal,wei2018more}. For this, we deliberately design a suboptimality gap, which is non-trival and quite different from the commonly defined suboptimality gap. We also build a new form of self-bounding constraint for PBM based on the property of the proposed suboptimality gap. Also, the structure of PBM could have Tsallis entropy as the potential function and the corresponding regularized leader can be computed efficiently, compared with the hybrid regularizer in the previous work \cite{zimmert2019beating}. We prove our algorithm could achieve $O(\log T)$ regret in the stochastic environment and $O(\sqrt{T})$ regret in the adversarial environment, verifying its simultaneous learning ability in both environments. 

Furthermore, we provide a regret lower bound $\Omega(m\sqrt{nT})$ for OLTR under PBM. This improves the state-of-the-art lower bound $\Omega(\sqrt{mnT})$ \cite{lattimore2018toprank} which is analyzed under document-based model, a special case of PBM. Our lower bound matches our upper bound, showing the optimality of both. 
Table \ref{table:full comparisons with related work} shows a full comparison of our work with most related works. 

The experiments show that our algorithm outperforms the baselines in adversarial environments while is competitive with \texttt{TopRank} \cite{lattimore2018toprank} and \texttt{PMED} \cite{komiyama2017position} in stochastic environments. The results show the simultaneous learning ability of our algorithm in both environments.

\begin{table*}[tbh!]

\centering
\small
\renewcommand{\arraystretch}{1}\
\begin{tabularx}{\textwidth}{@{}XXXp{3cm}}
\toprule
&Regret Bound \newline (Stochastic) &Regret Bound \newline(Adversarial) &Original Model \\
\midrule
\citet{kale2010non} &- &$\displaystyle O\left(m\sqrt{nT\log(n)}\right)$  &Bandits for Ordered Slates \\
\midrule
\citet{bubeck2012regret} &- &$\displaystyle O\left(m\sqrt{nT}\right)$  & CMAB with semi-bandit feedback \\
\midrule
\citet{lagree2016multiple} &$\displaystyle O\left(\frac{n}{\beta_m\Delta}\log(T)\right)$  &-  &PBM with known position bias \\
\midrule
\citet{zoghi2017online} &$\displaystyle O\left(\frac{m^3 n}{\Delta}\log(T)\right)$ &- &General Click Model\\
\midrule
\citet{lattimore2018toprank}&$\displaystyle O\left(\frac{m n}{\Delta}\log(T)\right)$ \newline $\displaystyle  O\left(\sqrt{m^3 n T \log(T)}\right)$ \newline $\displaystyle \Omega\left(\sqrt{mnT}\right)$ 
&$\displaystyle \Omega\left(\sqrt{mnT}\right)$ &General Click Model\\
\midrule
\citet{li2019online} &$\displaystyle O\left(m \sqrt{nT\log(nT)}\right)$ &- &General Click Model with Linear Features \\
\midrule
\citet{zimmert2019beating} &$\displaystyle O\left(\frac{m^2n}{\Delta_\beta \Delta}\log(T)\right)$ &$\displaystyle O\left(m\sqrt{nT}\right)$ &CMAB with semi-bandit feedback \\
\midrule
Ours &$\displaystyle O\left(\frac{m n}{\Delta_\beta \Delta}\log(T)\right)$
&$\displaystyle O\left(m\sqrt{nT}\right)$ \newline $\Omega\left(m\sqrt{nT}\right)$ & \\
\bottomrule
\end{tabularx}
\caption{This table compares regret bounds of related works when their results are applied to OLTR under PBM for both stochastic and adversarial environments. $T$ is the number of total rounds, $m$ is the number of positions and $n$ is the number of items. 
$\Delta$ is the minimal gap between the attractiveness of the best $m$ items. $\Delta_\beta$ is the minimal gap between position examination probabilities.
}\label{table:full comparisons with related work}
\end{table*}

\paragraph{Related Work}
The study of OLTR under PBM has received many interests. For the stochastic environment, \citet{lagree2016multiple} studies PBM but assumes the position examination probabilities are known or could be pre-computed from historical data. This assumption is a bit unrealistic and does not account for possible drift of position bias. \citet{komiyama2017position} remove this requirement but only provides an asymptotic regret bound. With rank-1 structure, PBM with unknown position bias can also be solved using methods in rank-1 bandits \cite{katariya2017stochastic} though their setting is originally designed to select one item each round. Some works study a general class of click models with PBM as a special case \cite{zoghi2017online,lattimore2018toprank,li2019online}. They distill a set of assumptions that are satisfied by common click models including the cascade model and PBM. The algorithms designed on this general click model are more robust than that on PBM. All the above algorithms study only the stochastic environment and might be brittle when the stochastic assumption is violated.

For the adversarial environment, PBM is first studied by \citet{kale2010non} as an ordered slate model. They solve it by a variant of multiplicative-weights algorithm and prove a regret upper bound $O(m\sqrt{n\log(n)T})$, $O(\sqrt{\log(n)})$ worse than ours. \citet{bubeck2012regret} show that OSMD with a $0$-potential function can achieve $O(m\sqrt{nT})$ regret, but their method need to know the time horizon. \citet{radlinski2008learning} study a ranked bandit problem using the greedy idea to select items one-by-one, which can only give approximation guarantees. It is extended to metric space by considering item contexts \cite{slivkins2013ranked}. Other studies include online optimization over the permutahedron \cite{ailon2014improved,ailon2016bandit}, which corresponds to PBM with $m=n$ and PBM with full-information feedback \cite{cohen2015following}. 

The model of OLTR under PBM can be regarded as a special case of combinatorial semi-bandits \cite{gai2012combinatorial,chen2013combinatorial,kveton2015tight,combes2015combinatorial,combes2017minimal,zimmert2019beating,neu2013efficient,neu2015first,audibert2014regret} with specific combinatorial constraints. Most existing works study either stochastic or adversarial environment. 

For the BOBW algorithms, many study this topic for MAB \cite{bubeck2012best,seldin2014one,auer2016algorithm,seldin2017improved,wei2018more} where \citet{zimmert2019optimal} show that FTRL with $\frac{1}{2}$-Tsallis entropy can achieve optimal regret bounds for both stochastic and adversarial environments. \citet{zimmert2019beating} study combinatorial semi-bandits by a novel hybrid regularizer but only show the optimality in two special cases for the stochastic environment, full combinatorial set and $m$-set. Other BOBW works include prediction with expert advice \cite{koolen2016combining,mourtada2019optimality}, linear bandits \cite{lee2021achieving}, online convex optimization \cite{cutkosky2017stochastic} and Markov decision process \cite{jin2020simultaneously}. Our work focuses on the BOBW under PBM.

\section{Setting}\label{sec:setting2}

This section introduces both stochastic and adversarial environments of OLTR under PBM.

Suppose there are $n$ items with item set $[n]=\{1,2,\dots,n\}$ and $m$ positions ($m\leq n$). In each round $t$, the learner selects an ordered list $I_t=(i_{t,1},i_{t,2},\dots,i_{t,m})$ consisting of $m$ distinct items, where $i_{t,j}$ denotes the item placed at position $j$ in round $t$. Note that this problem can be formulated as a special case of combinatorial semi-bandits and the action list $I_t$ can be written as a subpermutation matrix $X_t\in\cX$ where
\begin{align*}
    \cX=\left\{X\in\{0,1\}^{n\times m}~\Bigg|~\sum_{i=1}^n X_{i,j}=1, \forall j\in[m];\right.\\
    \left.~ \sum_{j=1}^m X_{i,j}\leq 1, \forall i\in[n]\right\}
\end{align*}
is the action set and $X_{i,j}$ denotes whether to put item $i$ on position $j$.

\citet{lattimore2020bandit} introduce the adversarial setting of PBM as follows. For each round $t$ and position $k\in [m]$, the environment secretly chooses $S_{t,k}$ as subset of $[n]$. The reward of round $t$ is defined as $r_t=\sum_{k=1}^m \bOne{i_{t,k}\in S_{t,k}}$ where $I_t$ is the selected action list at time $t$. The feedback is the positions of the clicked items. Notice that this model can be reformulated as a combinatorial semi-bandit problem. At round $t$, the environment secretly chooses all loss matrices $\ell_t\in\{0,1\}^{n\times m}$ for any $t$ before the game where $\ell_{t,i,j} = 0$ means there is no loss (or there is a click) if placing item $i$ at position $j$ in round $t$. After selecting $I_t$, the algorithm receives a loss of $\inner{X_t}{\ell_t}$ and observes semi-bandit feedback $\ell_{t,i,j}$ for those $(i,j)$ such that $X_{t,i,j}=1$. The goal of the algorithm is to minimize the expected cumulative pseudo-regret
\begin{align}
    \label{eq:regret def}
    R(T) = \EE{}{\sum_{t=1}^T \inner{X_t - x^\ast}{\ell_t}}\,,
\end{align}
where $x^* \in \argmin_{x\in\mathcal{X}} \EE{}{\sum^T_{t=1} \inner{x}{\ell_t}}$ is the best action and the expectation is taken over the randomness of both the algorithm and the environment.

For the stochastic environment, each item $i\in[n]$ is associated with an (unknown) attractiveness $\alpha_i\in [0,1]$ and each position $j\in[m]$ is associated with an (unknown) examination probability $\beta_j\in (0,1]$. Without loss of generality, we assume $\alpha_1 > \alpha_2 > \cdots > \alpha_m >\alpha_{m+1} \ge \alpha_{m+2} \ge \cdots \ge \alpha_n$ and $\beta_1 > \beta_2 > \cdots > \beta_m > \beta_{m+1} = 0$. Let $\cH_{t}$ 
be the $\sigma$-algebra  containing all the history $(\ell_1, X_1, \ldots, \ell_t, X_t)$ by the end of round $t$. In the stochastic environment, all elements in the loss matrix $\ell_t$ are $\cH_{t-1}$-conditionally independent whose $(i,j)$-entry is drawn from Bernoulli distribution $\Ber(1-\alpha_i \beta_j)$ like previous works \cite{komiyama2017position,lagree2016multiple,chuklin2015click}. In this case, $x_{i,j}^\ast$ is actually $\delta_{i,j}$ which is $1$ if and only if $i=j$. The goal of the algorithm is also to minimize the expected cumulative pseudo-regret Eq. \eqref{eq:regret def}.

\paragraph{Notations} Throughout this paper, we use $I_j^*$ to denote the item selected by the best action $x^\ast$ at position $j$, or $x^\ast_{I_j^\ast,j}=1$.
For a given set $\mathcal{X}$, let $\mathds{1}\{\mathcal{X}\}$  be the indicator function and $\mathcal{I}_{\mathcal{X}}(x)$ be the characteristic function which is $\infty$ if $x\notin \mathcal{X}$ and $0$ otherwise. Let $\Conv(\cX)$ be the convex hull of $\cX$. 
We use $\vone_n$ to denote the $n$-dimensional vector whose entries are all $1$s.
The conditional expectation $\EE{}{\cdot \mid \cH_{t-1}}$ is abbreviated as $\EE{t}{\cdot}$. For the stochastic environment, let $\Delta=\min_{i\in[m]}\{\alpha_i-\alpha_{i+1}\}$ be the minimal gap between the attractiveness of top $m$ items and  $\Delta_\beta=\min_{j\in[m]}\{\beta_j-\beta_{j+1}\}$ be the minimal gap between any two position examination probabilities.
\section{Algorithm}\label{sec:algo}

\begin{algorithm}[tb]
\caption{FTRL-PBM}
\label{alg:ftrlpbm}
\textbf{Input}: $\cX$.
\begin{algorithmic}[1]
\STATE $\hL_0=\mathbf{0}_{n\times m}$, $\eta_t=1/(2\sqrt{t})$. \label{algline1}
\FOR{$t=1,\dots,T$}
    \STATE Compute 
    \begin{align*}
        x_t=\argmin_{x\in\Conv(\cX)}\langle x, \hL_{t-1}\rangle + \frac{1}{\eta_t}\Psi(x)\,;
    \end{align*} \label{algline3}
    \STATE Sample $X_t\sim P(x_t)$; \label{algline4}
    \STATE Observe $\ell_{t,i,j}$ for those $(i,j)$-th entries satisfying $X_{t,i,j}=1$;
    \STATE Compute the loss estimator $\hat{\ell}_{t,i,j}=\frac{\ell_{t,i,j} \cdot \bOne{X_{t,i,j}=1}}{\xtij}$; \label{algline6}
    \STATE Compute $\hat{L}_t=\hat{L}_{t-1}+\hl_t$. \label{algline7}
\ENDFOR
\end{algorithmic}
\end{algorithm}

This section presents our main algorithm, $\ftrlpbm$, in Algorithm \ref{alg:ftrlpbm} for both stochastic and adversarial environments under PBM. Our $\ftrlpbm$ algorithm follows the general follow-the-regularized leader (FTRL) framework, whose main idea is to follow the action which minimizes the regularized cumulative loss of the past rounds. Since the complete loss vectors cannot be observed in the bandit setting, usually an unbiased estimator $\hat{\ell}_{t}$ satisfying $\EE{t}{\hat{\ell}_{t}}=\ell_t$ would serve as a surrogate.

Specifically, our algorithm $\ftrlpbm$ keeps track of the cumulative estimated loss $\hL_{t}=\sum^{t}_{s=1} \hat{\ell}_s \in \mathbb{R}_+^{n \times m}$ and initializes it as a zero vector (line \ref{algline1}). At each round $t$, $\ftrlpbm$ first computes a regularized leader $x_t$ in the convex hull of the action set $\Conv(\cX)$ by
\begin{align} 
    \label{eq:main}
    x_t=\argmin_{x\in\Conv(\cX)} \inner{x}{\hL_{t-1}} + \frac{1}{\eta_t}\Psi(x)
\end{align}
where $\Psi(x)$ is the regularizer (line \ref{algline3}). Here we take the $1/2$-Tsallis entropy
\begin{align*}
    \Psi(x)= \sum_i-\sqrt{x_i} 
\end{align*}
as our regularizer, which is shown optimal for BOBW MAB \cite{zimmert2019optimal}. For BOBW semi-bandits, the optimal algorithm adopts a hybrid regularizer \cite{zimmert2019beating}, which  is complicated and may be not efficient for PBM. 



Then $\ftrlpbm$ samples an action $X_t \sim P(x_t)$ from $\cX$ (line \ref{algline4}) where $P(x_t)$ satisfies $\EE{X\sim P(x_t)}{X}=x_t$. We follow previous works \cite{kale2010non,helmbold2009learning} to construct $P(x)$. The method is to first complete matrix $x_t$ into  a doubly stochastic matrix $M_t$, which is a convex combination of permutation matrices by Birkhoff's theorem, and then decompose matrix $M_t$ into its convex combination of at most $n^2$ permutation matrices by Algorithm 1 of \cite{helmbold2009learning}. The time complexity of whole sampling procedure is of order $O(n^{4.5})$ and the details can be found in Appendix \ref{app:opt}.

After observing the semi-bandit feedback for the selected action $X_t$, we can construct the unbiased estimator for the loss vector as
\begin{align*}
    \hat{\ell}_{t,i,j}=\frac{\ell_{t,i,j} \cdot \bOne{X_{t,i,j}=1}}{\xtij}
\end{align*}
for the $(i,j)$-th entry (line \ref{algline6}). Then the cumulative estimated loss $\hL$ is updated (line \ref{algline7}).


\subsection{Optimization}  \label{sec:cbp}

It remains to solve the constrained convex optimization problem \eqref{eq:main}. To avoid computing the expensive projection onto the feasible set, we consider the Frank-Wolfe (FW) \cite{frank1956algorithm} algorithm (a.k.a., conditional gradient method) due to its projection-free property. Specifically, the Frank-Wolfe algorithm only needs to compute the solution of a linear optimization over the feasible set in each iteration. We present the optimization algorithm for Problem \eqref{eq:main} in Algorithm \ref{alg:regLeader}. 

Notice that Algorithm \ref{alg:regLeader} requires to solve a linear optimization over $\Conv(\cX)$. This could be viewed as finding a maximal matching in a bipartite graph where edge $e_{i,j}$ has weight $r^{(k)}_{i,j}$ since $\Conv(\cX)$ is the convex hull of truncated permutation matrices and it could be solved by the Hopcroft–Karp algorithm \cite{hopcroft1973n} in $O(n^{2.5})$  time. Thus the total computational cost of Algorithm \ref{alg:regLeader} is $O(n^{2.5}K)$ time where $K$ is the maximal iteration number of Frank-Wolfe algorithm.



\begin{algorithm}[t]
\caption{Frank-Wolfe Algorithm for Problem \eqref{eq:main}}
\label{alg:regLeader}
\textbf{Input}: $x_{t-1}$, $\Conv(\cX)$, $\eta_t$, $\Psi(\cdot)$, $\hL_{t-1}$, $K$.

\begin{algorithmic}[1]
\STATE Let $x^{(1)}=x_{t-1}$, $f(\cdot)=\langle \cdot,\hL_{t-1}\rangle+\eta_t^{-1}\Psi(\cdot)$
\FOR{$k=1,\dots,K$}
    \STATE Compute $r^{(k)}:=\nabla f(x^{(k)})$.
    \STATE Compute $s^{(k)}:=\argmin_{s\in \Conv(\cX)}\langle s,r^{(k)} \rangle$. 
    \STATE Let $\gamma:=\frac{2}{1+k}$.
    \STATE Update $x^{(k+1)}:=(1-\gamma)x^{(k)}+\gamma s^{(k)}$.
\ENDFOR
\RETURN $x^{(K)}$.
\end{algorithmic}
\end{algorithm}

\section{Regret Analysis}
\label{sec:regret}

This section provides regret upper bounds of our algorithm $\ftrlpbm$ for both stochastic and adversarial environments, together with an improved lower bound for PBM, which also matches our upper bound. 
We also discuss the relationship between our results and previous works.

\subsection{Upper Bounds}

We give the regret upper bounds of $\ftrlpbm$ for each of the adversarial and stochastic environments and provide proof sketches.

\begin{theorem} \label{thm:adv}
For the adversarial environment, the regret of our $\ftrlpbm$ algorithm satisfies 
\begin{align*}
    R(T)&\leq 3m + 2m\log T + \sum_{t=1}^T\left( \frac{3}{\sqrt{t}}\sum_{j=1}^m\sum_{i\neq I^*_j} \sqrt{\BE[\xtij]}\right)\\
    &= O(m\sqrt{nT})\,.
\end{align*}
\end{theorem}

Though this regret bound matches that of OSMD with $0$-potential \cite{bubeck2012regret} and BOBW semi-bandits \cite{zimmert2019beating}, these two methods have some shortcomings compared to ours. OSMD with $0$-potential  needs to know the time horizon. Existing doubling trick methods lead to additional logarithmic factors in either stochastic or adversarial setting \cite{besson2018doubling}. BOBW semi-bandits could be inefficient under PBM since they use a hybrid regularizer.

\begin{proof}[Proof sketch]
Denote $\Psi_t(\cdot) = \frac{1}{\eta_t} \Psi(\cdot)$. Let
$\Phi_t(\cdot) = \max_{x\in\Conv(\cX)}\ \langle x, \cdot\rangle - \Psi_t(x)$  be the Fenchel conjugate of $\Psi_t+\cI_{\Conv(\cX)}$. Like the standard FTRL analysis (Chapter 28 of \cite{lattimore2020bandit}), the regret can be decomposed as a sum of the stability term and the regularization penalty term
\begin{align*}
    R(T)=&\underbrace{\BE\left[\sum_{t=1}^T\langle X_t,\ell_t\rangle+\Phi_t(-\hL_t)-\Phi_t(-\hL_{t-1})\right]}_{R_{stab}}\\
    &+\underbrace{\BE\left[\sum_{t=1}^T-\Phi_t(-\hL_t)+\Phi_t(-\hL_{t-1})-\langle x^*,\ell_t\rangle\right]}_{R_{pen}}\,.
\end{align*}
Then we bound these two terms separately (Lemma \ref{lem:stability} and Lemma \ref{lem:penalty} in the Appendix \ref{app:upperBound}) 
\begin{align}
    R_{stab}\leq& 3m + 2m\log T\notag\\ &+ \sum_{t=4}^T\left[ \frac{1}{\sqrt{t}}\sum_{j=1}^m\sum_{i\neq I^*_j} \left(\sqrt{\BE[\xtij]}+\BE[\xtij]\right)\right]\,, \label{eq:stab0}\\
    R_{pen} \leq&  \sum_{t=1}^T\sum_{j=1}^m \sum_{i\neq I^*_j}\frac{1}{\sqrt{t}}\left(2\sqrt{\BE[\xtij]}-\BE[\xtij]\right)\,. \label{eq:pen0}
\end{align}
Summing these two inequalities leads to the resulting regret upper bound. The second bound comes since $\sum_{i=1}^n\sqrt{\BE[\xtij]}\leq \sqrt{n}$.
\end{proof}

For the stochastic environment, it is key to prove a \textit{self-bounding} constraint like previous works \cite{zimmert2019beating,zimmert2019optimal,wei2018more}. 
The common suboptimality gap of putting item $i$ at position $j$ is defined as $\Delta_{i,j} = \beta_j(\alpha_j-\alpha_i)$, the reward difference from the right item \cite{komiyama2017position}.
This could be negative for $i<j$. When this happens, a better item is put at position $j$. Then there must be some \textit{bad} item placed before position $j$.  We account for this situation and introduce a new suboptimality gap definition that is more suitable to PBM. 
\begin{definition}
\label{def:Delta}
For any $i\in [n]$ and $j\in [m]$, define
\begin{align*}
    \Delta_{i,j}=
    \begin{cases}
      (\beta_j-\beta_{j+1})(\alpha_j-\alpha_i)\quad &j<i\,, \\
      0  &j= i\, , \\
      (\beta_{j-1}-\beta_j)(\alpha_i-\alpha_j)\quad &j> i\,.
    \end{cases}
\end{align*}
\end{definition}
The key idea for this definition comes from the incurred minimal regret of misplacing items. For the case that item $i$ is put at position $j$ with $i<j$. The item $i$ is misplaced but is better than the right item at position $j$, which should be item $j$. This means there must be some \textit{bad} items misplaced at earlier positions. The optimistic case is that item $j$ is just put at one position ahead $j-1$. Switching item $i$ and item $j$ would give the minimal regret gap, i.e. $\Delta_{i,j}$ is defined as the difference between the reward of $(\cdots, i,j,\cdots)$ and $(\cdots, j,i,\cdots)$ where the only effectively involved positions are $j-1,j$. The case of $i>j$ is similar.

Now we can present the \textit{self-bounding} constraint for PBM based on this introduced suboptimality gap.

\begin{lemma} \label{lem:self}
For the stochastic environment, the regret satisfies
\begin{align*}
    R(T)\geq \frac{1}{2}\sum_{t=1}^T\sum_{i=1}^n\sum_{j=1}^m \Delta_{i,j}\BE[x_{t,i,j}]\,.
\end{align*}
\end{lemma}
We first present Lemma \ref{lem:self1} which reveals the property of the introduced suboptimality gap. The proof of Lemma \ref{lem:self1} is postponed to Appendix \ref{app:lem:self}. 
\begin{lemma} \label{lem:self1}
Let $i_1,i_2,\dots,i_m$ be any sequence chosen from $[n]$ without repetition. Then
\begin{align*}
    \sum_{j=1}^m(\alpha_j \beta_j-\alpha_{i_j}\beta_j)\ge \frac{1}{2} \sum_{j=1}^m\Delta_{i_j,j}\,.
\end{align*}
\end{lemma}
Then Lemma \ref{lem:self} follows immediately from Lemma \ref{lem:self1} by summing over the time horizon.
\begin{proof}[Proof of Lemma \ref{lem:self}]
Lemma \ref{lem:self1} implies that
\begin{align*}
    R(T)&=\sum_{t=1}^T\sum_{j=1}^m\BE[\beta_j\alpha_j-\beta_j\alpha_{I_{t,j}}]\geq\sum_{t=1}^T\sum_{j=1}^m\BE[\frac{1}{2}\Delta_{I_{t,j},j}]\\
    &=\frac{1}{2}\sum_{t=1}^T\sum_{i=1}^n\sum_{j=1}^m \Delta_{i,j}\BE[x_{t,i,j}]\,,
\end{align*}
which completes the proof.
\end{proof}
%

With the self-bounding constraint in Lemma \ref{lem:self}, we can obtain the following regret bound for the stochastic setting.

\begin{theorem} \label{thm:sto}
For the stochastic environment, the regret of $\ftrlpbm$ algorithm is upper bounded by
\begin{align*}
    R(T) &\le \Bigg(18\sum_{j=1}^m \sum_{\substack{i=1\\ i \ne j}}^n \frac{1}{\Delta_{i,j}} + 4m\Bigg)\log T + 6m\\
    &= O\Big(\frac{mn}{\Delta_\beta \Delta} \log(T)\Big) \,.
\end{align*}
\end{theorem}

This regret bound improves a factor of $O(m)$ over that of BOBW semi-bandits \cite{zimmert2019beating} which is designed for general combinatorial cases. Our regret upper bound is $O(\frac{1}{\Delta_\beta})$ worse than \cite{lattimore2018toprank} which studies only the stochastic environments.

\begin{proof}
\begin{align*}
    R(T)
    \leq& 2R(T) - \frac{1}{2}\sum_{t=1}^T\sum_{j=1}^m \sum_{\substack{i=1\\ i \ne j}}^n \BE[x_{t,i,j}]\Delta_{i,j}\\
    \le &\sum_{t=1}^T\sum_{j=1}^m \sum_{\substack{i=1\\ i \ne j}}^n \left(6\sqrt{\frac{\BE[\xtij]}{t}} - \frac{1}{2} \BE[x_{t,i,j}]\Delta_{i,j} \right)\\
    &+6m+ 4m\log T \\
    \leq& \sum_{t=1}^T\sum_{j=1}^m \sum_{\substack{i=1\\ i \ne j}}^n \frac{18}{\Delta_{i,j}t} + 4m\log T + 6m\\
    \leq& \Bigg(18\sum_{j=1}^m \sum_{\substack{i=1\\ i \ne j}}^n \frac{1}{\Delta_{i,j}} + 4m\Bigg)\log T + 6m\,,
\end{align*}
where the first inequality is by Lemma \ref{lem:self}, the second inequality is by Eq.\eqref{eq:stab0} and Eq.\eqref{eq:pen0}, and the third inequality is due to the AM–GM inequality.
\end{proof}


\subsection{Lower Bound}

We provide an improved lower bound for PBM and defer its proof to Appendix \ref{app:lower}.
\begin{theorem}\label{thm:lower}
Suppose that $n\geq \max\{m+3, 2m\}$ and $T\geq n$.
For any algorithm there exists an instance of OLTR under PBM such that
\begin{align*}
    R(T)\geq \frac{1}{16}m\sqrt{(n-m+1)T}\,.
\end{align*}
\end{theorem}

Our lower bound improves $O(\sqrt{m})$ over the state-of-the-art lower bound \cite{lattimore2018toprank} and matches our upper regret bound.



\section{Experiments} \label{sec:exp}

This section compares the empirical performances of our $\ftrlpbm$ with related baselines where \texttt{TopRank} \cite{lattimore2018toprank}, \texttt{PBM-PIE} \cite{lagree2016multiple}, \texttt{PMED} \cite{komiyama2017position} are designed for the stochastic environment and \texttt{RankedExp3} \cite{radlinski2008learning},
\texttt{MW} \cite{kale2010non} are designed for the adversarial environment. We do not include \cite{zimmert2019beating} since we could not find an efficient method for PBM with their hybrid regularizer. Since the vanilla \texttt{PBM-PIE} in \cite{lagree2016multiple} needs the knowledge of the position examination probabilities, we use a bi-convex optimization to estimate the  examination probabilities for \texttt{PBM-PIE} like \cite{komiyama2017position} rather than directly supplying. All parameters are kept the same as in their original papers. For all experiments, we use $n=10$ items and $m=5$ positions. 

We only present the results of experiments on synthetic data in this section,  The results of experiments on real-world data are deferred to Appendix \ref{app:real_data}. For the synthetic data,  we set the position examination probabilities to $\beta=(1,\frac{1}{2},\cdots,\frac{1}{5})$ which are commonly adopted in previous works \cite{wang2018position,li2019online}. The attractiveness of items are set as $\alpha=(0.95, 0.95-\Delta, 0.95-2\Delta,\cdots, 0.95-9\Delta)$. We consider two cases of $\Delta=0.03$ and $\Delta=0.01$.


\begin{figure*}[tbh!]
\centering
\includegraphics[width=0.88\textwidth]{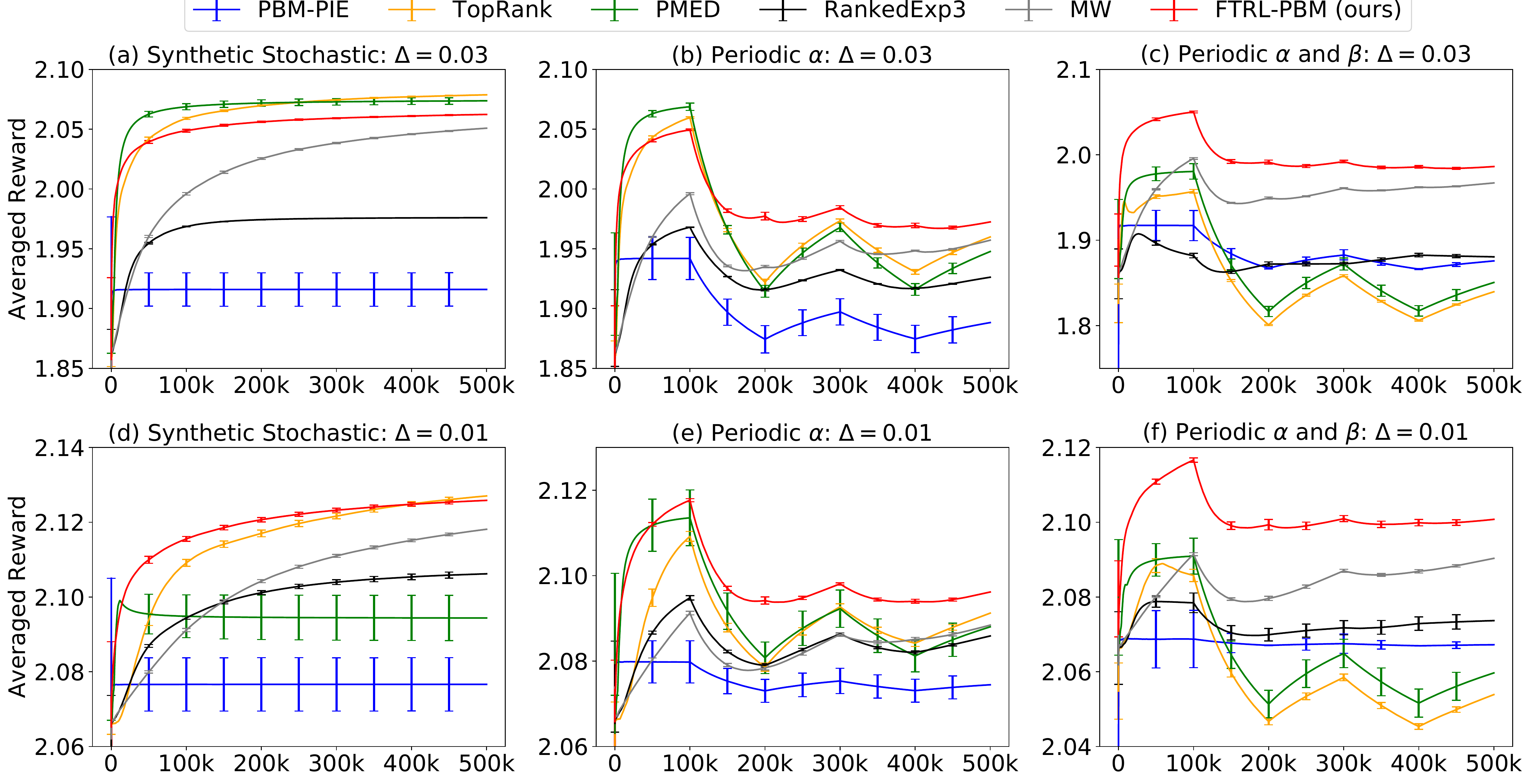}
\caption{This figure shows empirical comparisons between our $\ftrlpbm$ and \texttt{TopRank}, \texttt{PBM-PIE}, \texttt{PMED}, \texttt{RankedExp3} and \texttt{MW} in stochastic and periodic environments constructed by synthetic data. We adopt the metric of averaged rewards, which are the cumulative rewards divided by time $t$. All results are averaged over $10$ random runs and the error bars are standard errors which are standard deviations divided by $\sqrt{10}$.
}
\label{fig:exp}
\end{figure*}

We first construct stochastic environments \cite{lagree2016multiple,komiyama2017position} from the item attractiveness and position examination probabilities set above. The results are shown in Fig.\ref{fig:exp}(a)(d). 

Next we construct adversarial environments. Since it is a bit hard to design a real adversarial environment, we construct two periodical stochastic environments like \cite{zimmert2019optimal,zimmert2019beating}\footnote{They assume the relative order of items is fixed and the phase length is increasing.}.
We divide the whole time horizon into phases with $100$k rounds in each phase. For the first periodic environment, the position examination probabilities are fixed and the attractiveness of the first five items and last five items are exchanged periodically. Specifically, the odd phase uses the same environment as in the stochastic case and the even phase uses item attractiveness $(0.95-5\Delta, \cdots, 0.95-9\Delta, 0.95, \cdots, 0.95-4\Delta)$.
For the second periodic environment, the position examination probabilities are also changed periodically. We reverse the order of $\alpha,\beta$ simultaneously. Specifically, the odd phase uses the same environment as in the stochastic case and the even phase uses item attractiveness $(0.95-9\Delta, 0.95-8\Delta, \cdots, 0.95-\Delta, 0.95)$ and position examination probabilities $(\frac{1}{5}, \frac{1}{4}, \cdots, 1)$. 
The results are shown in Fig.\ref{fig:exp}(b)(c)(e)(f). 

\texttt{TopRank} performs best in (a)(d) since it is specially designed for the stochastic environment. \texttt{PMED} is also designed for the stochastic environment and has almost the same performance as \texttt{TopRank} in (a). Since it needs to solve a bi-convex optimization problem, fixed iterations would not give good convergence. Thus its performance has a large variance and deteriorates a lot for small gap (d). The design of \texttt{PBM-PIE} needs the knowledge of position bias. Though we can estimate them by bi-convex optimization, the estimation error would be amplified in the results when the estimated values are adopted directly. Then the performance of \texttt{PBM-PIE} is not very good and has a large variance in (a)(d). All of \texttt{TopRank}, \texttt{PMED} and \texttt{PBM-PIE} are strongly affected by the periodic changes (b)(c)(e)(f). 

Our algorithm $\ftrlpbm$ is competitive in stochastic environments and is best in  adversarial environments, showing the robustness and simultaneous learning ability of our algorithm. The greedy idea in \texttt{RankedExp3} is not very suitable for PBM but since it is designed for the adversarial environment, its performances are relatively stable. \texttt{MW} can be regarded as FTRL with negative Shannon entropy. It has good performances in some scenarios but is not the best due to the suboptimality of the regularizer.

\section{Conclusions}
\label{sec:conclusions}
To characterize the dynamic changes of online learning to rank (OLTR) environments, we study how to simultaneously learn in both stochastic and adversarial environments for OLTR under the position-based model (PBM). We design an algorithm based on the follow-the-regularized-leader framework and prove its efficiency in both environments. We also provide a lower bound for adversarial PBM which matches our upper bound. Experiments also validate the robustness of our algorithm.

Our results only focus on PBM. It would be a practical and promising topic to design efficient algorithms for both environments under general click models subsuming multiple click models. Further, the adversarial setting for general click models is open and suggested to be solved in the future.

\section{Acknowledgments}
The corresponding author Shuai Li is supported by National Natural Science Foundation of China
(62006151, 62076161). This work is sponsored by Shanghai Sailing Program. Cheng Chen is supported by Singapore Ministry of Education (AcRF) Tier 2 grant MOE2018-T2-1-013.

\bibliography{bibfile}
\appendix

\onecolumn

\setcounter{secnumdepth}{2}
\section{Proof of Lemma \ref{lem:self1}}\label{app:lem:self}
This section presents the proof of Lemma \ref{lem:self1}, which lower bounds the instantaneous regret in each round of the stochastic setting.



\begin{proof}[Proof of Lemma \ref{lem:self1}]
We prove it by induction on $m$. For $m=1$,
\begin{align*}
    \alpha_1\beta_1 - \alpha_{i_1}\beta_1 =\beta_1(\alpha_1-\alpha_{i_1}) \ge (\beta_1-\beta_2)(\alpha_1-\alpha_{i_1})=\Delta_{i_1,1} \ge \frac{1}{2} \Delta_{i_1,1}\,.
\end{align*}
Suppose the conclusion holds for $m-1$. Now we consider the situation of $m$.

If $i_1=1$, by the induction hypothesis, there is
\begin{align*}
    \sum_{j=1}^m(\alpha_j\beta_j-\alpha_{i_j}\beta_j)=\sum_{j=2}^m(\alpha_j\beta_j-\alpha_{i_j}\beta_j)\ge \frac{1}{2} \sum_{j=2}^m\Delta_{i_j,j} = \frac{1}{2} \sum_{j=1}^m\Delta_{i_j,j}\,.
\end{align*}
If $i_1 \neq 1$, it suffices to consider the following two cases.

Case (a). If $i_j\neq 1$ for all $j\in[m]$, we define a sequence $(i_j')_{j=1}^m$ by $i_1'=1$ and $i_j'=i_j$ for $j\geq 2$. 
By induction, one can see that
\begin{equation} \label{eq:1}
\sum_{j=1}^m(\alpha_j\beta_j-\alpha_{i_j'}\beta_j)=\sum_{j=2}^m(\alpha_j\beta_j-\alpha_{i_j'}\beta_j)\geq \frac{1}{2}\sum_{j=2}^m \Delta_{i_j,j}\, .
\end{equation}
Since $\alpha_1\geq \alpha_{i_1}$, it is clear that
\begin{equation} \label{eq:2}
\alpha_{i_1'}\beta_1-\alpha_{i_1}\beta_1=\beta_1(\alpha_1-\alpha_{i_1})\geq (\beta_1-\beta_2)(\alpha_1-\alpha_{i_1})=\Delta_{i_1,1}\, .
\end{equation}
Adding Eq.\eqref{eq:1} and Eq.\eqref{eq:2} shows that
\begin{align*}
    \sum_{j=1}^m(\alpha_j\beta_j-\alpha_{i_j}\beta_j)\geq \Delta_{i_1,1}+\frac{1}{2}\sum_{j=2}^m \Delta_{i_j,j}\geq \frac{1}{2}\sum_{j=1}^m \Delta_{i_j,j}\, . 
\end{align*}

Case (b). If there exists $k\in\{2,3,\dots,m\}$ such that $i_k=1$, we define a sequence $(i_j')_{j=1}^m$ by $i_1'=1$, $i_k'=i_1$  and $i_j'=i_j$ for $j\neq 1,k$. 
By induction, one can see that
\begin{equation} \label{eq:3}
\begin{split}
    \sum_{j=1}^m(\alpha_j\beta_j-\alpha_{i_j'}\beta_j)=& \sum_{j=2}^m(\alpha_j\beta_j-\alpha_{i_j'}\beta_j)\\
    \geq& \frac{1}{2}\sum_{j=2}^m\Delta_{i_j',j}= \frac{1}{2}\sum_{j=1}^m\Delta_{i_j,j}-\frac{1}{2}(\Delta_{i_1,1}+\Delta_{1,k}-\Delta_{i_1,k})\, .
\end{split}
\end{equation}
Also, it holds that
\begin{equation} \label{eq:4}
\sum_{j=1}^m(\alpha_{i_j'}\beta_j-\alpha_{i_j}\beta_j)=\beta_1(\alpha_1-\alpha_{i_1})+\beta_k(\alpha_{i_1}-\alpha_1)= (\beta_1-\beta_k)(\alpha_1-\alpha_{i_1})\, .
\end{equation}
Adding Eq.\eqref{eq:3} and Eq.\eqref{eq:4} shows that
\begin{align*}
    \sum_{j=1}^m(\alpha_j\beta_j-\alpha_{i_j}\beta_j)\geq \frac{1}{2}\sum_{j=1}^m\Delta_{i_j,j}+(\beta_1-\beta_k)(\alpha_1-\alpha_{i_1})-\frac{1}{2}(\Delta_{i_1,1}+\Delta_{1,k}-\Delta_{i_1,k})\,,
\end{align*}
which remains to show $(\beta_1-\beta_k)(\alpha_1-\alpha_{i_1})\geq \frac{1}{2}(\Delta_{i_1,1}+\Delta_{1,k}-\Delta_{i_1,k})$. We consider the following three cases.
\begin{enumerate}
    \item If $k=i_1$, then $\Delta_{i_1,k}=0$ and
        \begin{align*}
            \frac{1}{2}(\Delta_{i_1,1}+\Delta_{1,k}-\Delta_{i_1,k})=& \frac{1}{2}(\beta_1-\beta_2+\beta_{k-1}-\beta_k)(\alpha_1-\alpha_k)\\
            \leq& (\beta_1-\beta_k)(\alpha_1-\alpha_k)=(\beta_1-\beta_k)(\alpha_1-\alpha_{i_1})\, .
        \end{align*}
    \item If $k<i_1$, then $\alpha_k\geq\alpha_{i_1}$ and
        \begin{align*}
        \frac{1}{2}(\Delta_{i_1,1}+\Delta_{1,k}-\Delta_{i_1,k})\leq& \frac{1}{2}(\Delta_{i_1,1}+\Delta_{1,k})\\
        =&\frac{1}{2}((\beta_1-\beta_2)(\alpha_1-\alpha_{i_1})+(\beta_{k-1}-\beta_k)(\alpha_1-\alpha_{k}))\\
        \leq& \frac{1}{2}((\beta_1-\beta_2)(\alpha_1-\alpha_{i_1})+(\beta_{k-1}-\beta_k)(\alpha_1-\alpha_{i_1}))\\
        \leq& (\beta_1-\beta_k)(\alpha_1-\alpha_{i_1})\, .
        \end{align*}
    \item If $k>i_1$, then $\alpha_k\leq\alpha_{i_1}$ and
        \begin{align*}
        &\frac{1}{2}(\Delta_{i_1,1}+\Delta_{1,k}-\Delta_{i_1,k})\\
        =&\frac{1}{2}((\beta_1-\beta_2)(\alpha_1-\alpha_{i_1})+(\beta_{k-1}-\beta_k)(\alpha_1-\alpha_{k})-(\beta_{k-1}- \beta_k)(\alpha_{i_1}-\alpha_k))\\
        =&\frac{1}{2}((\beta_1-\beta_2)(\alpha_1-\alpha_{i_1})+(\beta_{k-1}-\beta_k)(\alpha_1-\alpha_{i_1}))\leq (\beta_1-\beta_k)(\alpha_1-\alpha_{i_1})\, .    
        \end{align*}
\end{enumerate}



To sum up, we can achieve
\begin{align*}
    \sum_{j=1}^m(\alpha_j\beta_j-\alpha_{i_j}\beta_j)\geq \frac{1}{2}\sum_{j=1}^m\Delta_{i_j,j}\, .
\end{align*}

\end{proof}



\section{Proof of Theorem \ref{thm:adv}} \label{app:upperBound}
In this section, we provide the details for the proof of Theorem \ref{thm:adv}. Before presenting the proof, we introduce some definitions and tools from the convex analysis in section \ref{sec:pre}. Then we analyze the regularization penalty term and the stability term in section \ref{sec:pen} and section \ref{sec:stab}, respectively. 

\subsection{Preliminaries} \label{sec:pre}
Define $\tPsi(\cdot)=\Psi(\cdot)+\langle \cdot, \vone\rangle$ and $\tPsi_t(\cdot)=\eta_t^{-1}\tPsi(\cdot)$. Let $\bar{\BR}=\BR\bigcup\{-\infty, +\infty\}$ be the extended real number system. We can extend the range of $\tPsi$ to $\bar{\BR}$ by setting $\tPsi(x)=\infty$ for any $x\in \BR^{n\times m}\setminus\cD$, where $\cD=[0,+\infty)^{n\times m}$ is the domain of $\tPsi$. Then $\tPsi$ is a Legendre function \cite{rockafellar2015convex}.

The Fenchel conjugate of a convex function $f$ is defined as
\begin{align*}
    f^*(\cdot)=\max_{x\in\BR^d}\inner{x}{\cdot}-f(x)\,.
\end{align*}
Then for any $y\leq 0$, the Fenchel conjugate of $\tPsi_t(\cdot)$ is
\begin{align*}
    \tPsi_t^*(y) = \max_{x\in \cD}\langle x,y \rangle-\tPsi_t(x)=\sum_{i=1}^n\sum_{j=1}^m \frac{1}{4\eta_t(1-\eta_ty_{i,j})}\,.
\end{align*}

Let $\tPhi_t(\cdot) = \max_{x\in\Conv(\cX)}\ \langle x, \cdot\rangle - \tPsi_t(x)$. It is clear that $\tPhi_t(\cdot)$ is the Fenchel conjugate of $\tPsi_t+\cI_{\Conv(\cX)}$. 
According to Section 26 of \cite{rockafellar2015convex}, the following properties hold
\begin{align}
    \nabla\tPsi_t&=(\nabla\tPsi^*_t)^{-1}\,,\label{eq:legendre}\\
    \nabla\tPhi_t(\cdot)&=\argmax_{x\in\Conv(\cX)}\inner{x}{\cdot}-\tPsi_t(x)\,, \nonumber \\
    \nabla\tPsi^*_t(\cdot)&=\argmax_{x\in\cD}\inner{x}{\cdot}-\tPsi_t(x)\,, \nonumber\\
    \nabla\Phi_t(\cdot)&=\argmax_{x\in\Conv(\cX)}\inner{x}{\cdot}-\Psi_t(x)\,. \nonumber
\end{align}
Since $\inner{x}{\vone}=m$ for any $x\in\Conv(\cX)$,
we know that 
\begin{align} \label{eq:dual}
\nabla\tPhi_t(-\hat{L}_{t-1})=\nabla\Phi_t(-\hat{L}_{t-1})=x_t,
\end{align}
which is the regularized leader of FTRL.


The Bregman divergence associated with a Legendre function $f$ is defined as
\begin{align*}
    D_f(x,y)=f(x)-f(y)-\inner{\nabla f(y)}{x-y}\,.
\end{align*}

\begin{lemma}[Lemma 4 of \cite{zimmert2019beating}]\label{bregmanDiv1}
    For any $L\in(-\infty,0]^{n\times m}$, let $\tL=\nabla\tPsi_t(\nabla\tPhi_t(L))$. Then for any $\ell\in(-\infty,0]^{n\times m}$, it holds that
    \begin{align*}
        D_{\tPhi_t}(L+\ell,L)\leq D_{\tPsi_t^*}(\tL+\ell,\tL)\,.
    \end{align*}
\end{lemma}

\subsection{Regularization penalty term} \label{sec:pen}

\begin{lemma} \label{lem:penalty}
The regularization penalty term can be bounded as
\begin{align*}
    R_{pen} \leq  \sum_{t=1}^T\sum_{j=1}^m \sum_{i\neq I^*_j}\frac{1}{\sqrt{t}}\left(2\sqrt{\BE[\xtij]}-\BE[\xtij]\right)\,.
\end{align*}
\end{lemma}

\begin{proof}

Recall $\Phi_t(\cdot)$ is defined as $\Phi_t(\cdot) = \max_{x\in\Conv(\cX)}\ \langle x, \cdot\rangle - \Psi_t(x)$. One can see that

\begin{align}
        & \sum_{t=1}^T \left( -\Phi_t(-\hL_t)+\Phi_t(-
\hL_{t-1})-\langle x^*,\hl_t\rangle \right)\notag\\
=& \sum_{t=1}^T\left(\min_{x\in\Conv(\cX)}\left\{\langle x,\hL_t \rangle+\eta_t^{-1}\Psi(x)\right\}-\min_{x\in\Conv(\cX)}\left\{\langle x,\hL_{t-1} \rangle+\eta_t^{-1}\Psi(x)\right\}-\langle x^*,\hl_t\rangle\right)\notag\\
=& \sum_{t=1}^T\left(\min_{x\in\Conv(\cX)}\left\{\langle x,\hL_t \rangle+\eta_t^{-1}\Psi(x)\right\}-\left(\langle x_t,\hL_{t-1} \rangle+\eta_t^{-1}\Psi(x_t)\right)-\langle x^*,\hl_t\rangle\right)\notag\\
\leq& \langle x^*, \hL_T\rangle +\eta_T^{-1}\Psi(x^*)+\sum_{t=1}^{T-1}\left(\langle x_{t+1},\hL_{t} \rangle+\eta_t^{-1}\Psi(x_{t+1})\right)-\sum_{t=1}^T\left(\langle x_t,\hL_{t-1} \rangle+\eta_t^{-1}\Psi(x_t)\right)-\langle x^*,\hL_T\rangle\notag\\
=& \eta_T^{-1}\Psi(x^*)+\sum_{t=1}^{T-1}\eta_t^{-1}\Psi(x_{t+1})-\sum_{t=1}^T\eta_t^{-1}\Psi(x_t)\notag\\
=& \eta_1^{-1}(\Psi(x^*)-\Psi(x_1))+\sum_{t=2}^T(\eta_t^{-1}-\eta_{t-1}^{-1})(\Psi(x^*)-\Psi(x_t))\notag\\
=& 2(\Psi(x^*)-\Psi(x_1))+2\sum_{t=2}^T(\sqrt{t}-\sqrt{t-1})(\Psi(x^*)-\Psi(x_t))\,.\label{eq:lem4}
\end{align}
where the third equation comes from the fact $\hL_0=0$.\\
Using $\BE[\ell_t]=\BE[\hat{\ell}_t]$ and taking expectations of both sides of Eq.\eqref{eq:lem4} lead to
\begin{align}\label{eq:pen5}
    R_{pen} \leq \EE{}{2(\Psi(x^*)-\Psi(x_1))+2\sum_{t=2}^T(\sqrt{t}-\sqrt{t-1})(\Psi(x^*)-\Psi(x_t))}\,.
\end{align}

Since $\Psi(x)= \sum_i-\sqrt{x_i}$, it holds that
\begin{align}
        \Psi(x^*)-\Psi(x_t) &=\sum_{i=1}^n\sum_{j=1}^m\sqrt{\xtij}-\sum_{j=1}^m\sum_{i:i=I^*_j}\sqrt{1}\notag\\
    &=\sum_{i=1}^n\sum_{j=1}^m(\sqrt{\xtij}-\frac{1}{2}\xtij)-\sum_{j=1}^m\sum_{i:i=I^*_j}(\sqrt{1}-\frac{1}{2})\notag\\
    &\leq \sum_{j=1}^m \sum_{i:i\neq I^*_j}(\sqrt{\xtij}-\frac{1}{2}\xtij)\,,\label{eq:eq9}
\end{align}
where the last inequality is due to $\sqrt{x_{t,i,j}}-\frac{x_{t,i,j}}{2}\leq\frac{1}{2}$ for any $0\leq x_{t,i,j}\leq 1$.

Substituting Eq.\eqref{eq:eq9} into Eq.\eqref{eq:pen5} shows that
\begin{align*}
    R_{pen} &\leq \BE\left[2(\Psi(x^*)-\Psi(x_1))+2\sum_{t=2}^T(\sqrt{t}-\sqrt{t-1})(\Psi(x^*)-\Psi(x_t))\right]\\
    &\leq \BE\left[\sum_{t=1}^T\frac{2}{\sqrt{t}}(\Psi(x^*)-\Psi(x_t))\right]\\
    &\leq \BE\left[\sum_{t=1}^T\sum_{j=1}^m \sum_{i:i\neq I^*_j}\frac{2}{\sqrt{t}}(\sqrt{\xtij}-\frac{1}{2}\xtij)\right]\\
    &\leq \sum_{t=1}^T\sum_{j=1}^m \sum_{i:i\neq I^*_j}\frac{1}{\sqrt{t}}(2\sqrt{\BE[\xtij]}-\BE[\xtij])\,,
\end{align*}
where the second inequality is due to $\sqrt{t}-\sqrt{t-1}\leq \frac{1}{\sqrt{t}}$ for any $t\geq1$ and the last step comes from Jensen’s inequality.
\end{proof}

\subsection{Stability term} \label{sec:stab}
\label{app:stab}

\begin{lemma} \label{lem:stability}
The stability term can be bounded as
\begin{align*}
    R_{stab}\leq 3m + 2m\log T + \sum_{t=4}^T\left[ \frac{1}{\sqrt{t}}\sum_{j=1}^m\sum_{i\neq I^*_j} \left(\sqrt{\BE[\xtij]}+\BE[\xtij]\right)\right]\,.
\end{align*}
\end{lemma}


\begin{lemma}\label{lem:stab1}
    For any $w_i\in\BR$,
    if $\eta_t(\ltij-w_i)\xtij^{\frac{1}{2}}\geq -\frac{1}{4}$, then
    \begin{align*}
        \langle x_t,\hl_t\rangle + \Phi_t(-\hL_t)-\Phi_t(-\hL_{t-1})\leq \sum_{i=1}^n\sum_{j=1}^m\left[2\eta_t\xtij^{\frac{3}{2}}(\ltij-w_i)^2+8\eta_t^2\xtij^2|w_i-\ltij|_+^3\right]\,,
    \end{align*}
    where $|z|_+=\max\{z,0\}$.
\end{lemma}


\begin{proof}
    For any $w=\begin{bmatrix}w_1\vone_n,w_2\vone_n,\cdots, w_m\vone_n\end{bmatrix}\in\BR^{n\times m}$, 
    we have
    \begin{align}\label{eq:stab3}
        &\langle x_t,\hl_t\rangle + \Phi_t(-\hL_t)-\Phi_t(-\hL_{t-1})\notag\\
        =&\langle x_t,\hl_t\rangle + \tPhi_t(-\hL_t)-\tPhi_t(-\hL_{t-1})\notag\\
        =&\langle x_t,\hl_t-w\rangle + \tPhi_t(-\hL_t+w)-\tPhi_t(-\hL_{t-1})\notag\\
        =&\langle x_t,\hl_t-w\rangle + \tPhi_t(-\hL_{t-1}-\hl_t+w)-\tPhi_t(-\hL_{t-1})\notag\\
        \overset{\text{(a)}}{=}&D_{\tPhi_t}(-\hL_{t-1}-\hl_t+w,-\hL_{t-1})\notag\\
        \leq& D_{\tPsi_t}(\nabla\tPsi_t(x_t)-\hl_t+w,\nabla\tPsi_t(x_t))\notag\\
        \overset{\text{(b)}}{=}& \langle x_t,\hl_t-w\rangle + \tPsi_t^*(\nabla\tPsi_t(x_t)-\hl_t+w)-\tPsi_t^*(\nabla\tPsi_t(x_t))\notag\\
        =& \sum_{i=1}^n\sum_{j=1}^m\left[ \xtij(\ltij-w_i)+\frac{1}{2\eta_t}(\xtij^{-\frac{1}{2}}+2\eta_t(\ltij-w_i))^{-1}-\frac{1}{2\eta_t}\xtij^{\frac{1}{2}}\right]\notag\\
        =& \sum_{i=1}^n\sum_{j=1}^m\left[\frac{1}{2\eta_t}\xtij^{\frac{1}{2}}\left(2\eta_t\xtij^{\frac{1}{2}}(\ltij-w_i) + (1+2\eta_t(\ltij-w_i)\xtij^{\frac{1}{2}})^{-1}-1\right)\right]\notag\\
        =& \sum_{i=1}^n\sum_{j=1}^m\left[2\eta_t\xtij^{\frac{3}{2}}(\ltij-w_i)^2(1+2\eta_t(\ltij-w_i)\xtij^{\frac{1}{2}})^{-1}\right]\,,
    \end{align}
where 
the first two equality comes from 
$\sum^n_{i=1}x_{t,i,j}=1$, equality (a) and (b) is by Eq.\eqref{eq:dual}
and the inequality is due to
Lemma \ref{bregmanDiv1}. 

Since we have $\eta_t(\ltij-w_i)\xtij^{\frac{1}{2}}\geq -\frac{1}{4}$ for any $i\in[n]$ and $j\in[m]$, it can be obtained that $(1+2\eta_t(\ltij-w_i)\xtij^{\frac{1}{2}})^{-1}\leq 2$. Thus, 
\begin{align} \label{eq:stab4}
    (1+2\eta_t(\ltij-w_i)\xtij^{\frac{1}{2}})^{-1}=& 1-2\eta_t(\ltij-w_i)\xtij^{\frac{1}{2}}(1+2\eta_t(\ltij-w_i)\xtij^{\frac{1}{2}})^{-1} \nonumber\\
    \leq& 1+4\eta_t|w_i-\ltij|_+\xtij^{\frac{1}{2}}\,.
\end{align}

Substituting Eq.(\ref{eq:stab4}) into Eq.(\ref{eq:stab3}) concludes the proof. 

\end{proof}

We are now ready to prove Lemma \ref{lem:stability}.
\begin{proof}[Proof of Lemma \ref{lem:stability}]
By the unbiasedness of $\hl_t$ and the tower rule, one can see that
\begin{align}
    \BE[\langle X_t,\ell_t\rangle]=\BE[\BE_t[\langle X_t,\ell_t\rangle]]=\BE[\BE_t[\langle x_t,\ell_t\rangle]]=\BE[\BE_t[\langle x_t,\hl_t\rangle]]\,.\label{eq:tower}
\end{align}

Let $w_j=\ell_{t,I_{t,j},j}$. Lemma \ref{lem:stab1} and Eq.\eqref{eq:tower} show that
\begin{align*}
    &\BE\left[\langle X_t,\ell_t\rangle + \Phi(-\hL_t)-\Phi(-\hL_{t-1})\right]\\
\leq& \BE\left[\sum_{i=1}^n \sum_{j=1}^m 2\eta_t\xtij^{\frac{3}{2}}(\ltij-\ell_{t,i,I_{t,i}})^2+8\eta_t^2\xtij^2|w_i-\ltij|_+^3 \right]\\
\leq& 2\eta_t\sum_{j=1}^m \BE\left[ x_{t,I_{t,j},j}^{\frac{3}{2}}(\hl_{t,I_{t,j},j}-\ell_{t,I_{t,j},j})^2+ \sum_{i:i\neq I_{t,j}} \xtij^{\frac{3}{2}}(\ltij-\ell_{t,I_{t,j},j})^2\right]+8\eta_t^2m\\
\leq& 2\eta_t\sum_{j=1}^m\BE\left[ x_{t,I_{t,j},j}^{-\frac{1}{2}}(1-x_{t,I_{t,j},j})^2+ \sum_{i:i\neq I_{t,j}} \xtij^{\frac{3}{2}}\right]+8\eta_t^2m\\
=& 2\eta_t\sum_{i=1}^n\sum_{j=1}^m\BE\left[ \xtij^{\frac{1}{2}}(1-\xtij)^2+ (1-\xtij)\xtij^{\frac{3}{2}}\right]+8\eta_t^2m\\
=& 2\eta_t\sum_{i=1}^n\sum_{j=1}^m\BE\left[ \xtij^{\frac{1}{2}}(1-\xtij)\right]+8\eta_t^2m\\
\leq& 2\eta_t\sum_{i=1}^n\sum_{j=1}^m \BE[\xtij]^{\frac{1}{2}}(1-\BE[\xtij])+8\eta_t^2m\,,
\end{align*}
where 
the second inequality is due to
$\sum^n_{i=1}\sum^m_{j=1}\xtij^2|w_i-\ltij|_+^3\leq m$.
The third inequality comes from $\ltij=0$ for $i\neq I_{t,j}$ and
\begin{align*}
    \left(\hl_{t,I_{t,j},j}-\ell_{t,I_{t,j},j}\right)^2&=\left(\frac{\ell_{t,I_{t,j},j}\mathds{1}{\{X_{t,I_{t,j},j}=1\}}}{\x_{t,I_{t,j},j}}-\ell_{t,I_{t,j},j}\right)^2\\
    &=\ell_{t,I_{t,j},j}^2\left(\frac{1}{\x_{t,I_{t,j},j}}-1\right)^2\\
    &\leq \left(\frac{1}{\x_{t,I_{t,j},j}}-1\right)^2\,.
\end{align*}
And the last inequality follows from Jensen’s inequality.

Thus,
\begin{align}
    &\BE\left[\sum_{t=1}^T\langle X_t,\ell_t\rangle+\Phi_t(-\hL_t)-\Phi_t(-\hL_{t-1})\right]\notag\\
    \leq& 3m + \BE\left[\sum_{t=4}^T\langle X_t,\ell_t\rangle+\Phi_t(-\hL_t)-\Phi_t(-\hL_{t-1})\right]\notag\\
    \leq& 3m + 2m\log T + \sum_{t=4}^T\left[ 2\eta_t\sum_{i=1}^n\sum_{j=1}^m \BE[\xtij]^{\frac{1}{2}}(1-\BE[\xtij])\right]\,.\label{eq:stab51}
\end{align}

If $i\neq I^*_j$,
\begin{align}
    \sqrt{\BE[\xtij]}(1-\BE[\xtij])\leq \sqrt{\BE[\xtij]}\,.\label{eq:sta61}
\end{align}
If $i=I^*_j$ ,
\begin{align}
    \sqrt{\BE[\xtij]}(1-\BE[\xtij])\leq 1-\BE[\xtij]=\sum_{i:i\neq I^*_j}\BE[\xtij]\,.\label{eq:sta71}
\end{align}

Substituting Eq.\eqref{eq:sta61} and Eq.\eqref{eq:sta71} into Eq.\eqref{eq:stab51} shows that
\begin{align*}
    R_{stab}\leq 3m + 2m\log T + \sum_{t=4}^T\left[ 2\eta_t\sum_{j=1}^m\sum_{i\neq I^*_j} (\sqrt{\BE[\xtij]}+\BE[\xtij])\right]\,,
\end{align*}
which concludes the proof.
\end{proof}

\subsection{Proof of Theorem \ref{thm:adv}}
Now we turn to the proof of Theorem \ref{thm:adv}.

\begin{proof}[Proof of Theorem \ref{thm:adv}]
Like the standard FTRL analysis (see Chapter 28 of \cite{lattimore2020bandit}), the regret can be decomposed as
\begin{align*}
    R(T)=&\underbrace{\BE\left[\sum_{t=1}^T\langle X_t,\ell_t\rangle+\Phi_t(-\hL_t)-\Phi_t(-\hL_{t-1})\right]}_{R_{stab}}+\underbrace{\BE\left[\sum_{t=1}^T-\Phi_t(-\hL_t)+\Phi_t(-\hL_{t-1})-\langle x^*,\ell_t\rangle\right]}_{R_{pen}}\,.\\
    \leq& 3m + 2m\log T + \sum_{t=1}^T\left( \frac{3}{\sqrt{t}}\sum_{j=1}^m\sum_{i\neq I^*_j} \sqrt{\BE[\xtij]}\right)\\
\end{align*}
where the first inequality comes from Lemma \ref{lem:penalty} and Lemma \ref{lem:stab1}. According to Cauchy-Schwartz inequality, we have $\sum_{i=1}^n\sqrt{\BE[\xtij]}\leq \sqrt{n}$. Thus,

\begin{align*}
    R(T)\leq 3m + 2m\log T + 3m\sqrt{n}\sum_{t=1}^T\frac{1}{\sqrt{t}} \leq 6m\sqrt{nT}
\end{align*}
where the last inequality comes from the fact 
$\sum_{t=1}^T\frac{1}{\sqrt{t}}\leq 2\sqrt{T}$.

\end{proof}

\section{Lower Bound}\label{app:lower}


\begin{proof}[Proof of Theorem \ref{thm:lower}]
We assume the learner is deterministic in this section.
Let $\cS_{n}(m)$ be the set of all $m$-permutations of $[n]$.
For each $u\in\cS_{n}(m)$, we define a loss vector $\ell_u$ as

\begin{align*}
    \ell_u(i,j)=
    \begin{cases}
      \frac{1}{2}-\Delta & \quad \text{if } u_j=i \\
      \frac{1}{2} & \quad \text{otherwise}\,,
    \end{cases}
\end{align*}
where $0<\Delta<1/2$ is some value to be tuned subsequently. 

Let $N_{i,j}(t) = \sum_{s=1}^t \bOne{X_{s,i,j}=1}$ be the total number of times that item $i$ is placed at position $j$ by the end of round $t$. We will use subscript $u$ to denote the expectation, probability, regret, etc., under the ranking problem whose loss is determined by $\ell_u$.
Note that $R_u(T)=\Delta\sum_{j=1}^m (T-\EE{u}{N_{u_j,j}(T)})$. Then
\begin{align*}
    \sum_{u\in\cS_{n}(m)}R_u(T)=& \Delta\sum_{u\in\cS_{n}(m)}\sum_{j=1}^m (T-\EE{u}{N_{u_j,j}(T)})\\
=& \Delta\sum_{j=1}^m \sum_{u_{-j}\in\cS_n(m-1)}  \sum_{u_j\in[n]\setminus u_{-j}}(T-\EE{u}{N_{u_j,j}(T)})\,,
\end{align*}
where $u_{-j}=(u_1,\cdots,u_{j-1},u_{j+1},\cdots,u_m)$ denotes a $(m-1)$-permutation of $[n]$.

Now fix $j\in[m]$, $u_{-j}\in\cS_n(m-1)$ and $T$.
Let $J_{j,T}$ be drawn according to the probability $\left(\frac{N_{1,j}(T)}{T},\frac{N_{2,j}(T)}{T},\dots,\frac{N_{n,j}(T)}{T}\right)$.
Let $\BP_{j,u}$ be the law of $J_{j,T}$ under the ranking problem whose loss is determined by $\ell_u$.
Then we have $\BP_{j,u}(J_{j,T}=i)=\EE{u}{\frac{N_{i,j}(T)}{T}}$.

We define a ranking problem with respect to $u_{-j}$ whose loss is determined by $\ell_{u_{-j}}$ where 
\begin{align*}
    \ell_{u_{-j}}(i,j^\prime)=
    \begin{cases}
       \ell_u(i,j^\prime) & \quad \text{if } j^\prime\neq j \\
      \frac{1}{2} & \quad \text{otherwise}\,.
    \end{cases}
\end{align*}
Similarly, we use subscript $u_{-j}$ to denote the expectation, probability, regret, etc., under 
the ranking problem whose loss is determined by $\ell_{u_{-j}}$. We also denote by $\BP_{j,u_{-j}}$ the law of $J_{j,T}$ when the learner is interacting with the ranking problem determined by $\ell_{u_{-j}}$.


By Pinsker's inequality
(Chapter 14 of \cite{lattimore2020bandit})
, we have
\begin{align*}
    \BP_{j,u}(J_{j,T}=u_j)\leq  \BP_{j,u_{-j}}(J_{j,T}=u_j)+\sqrt{\frac{1}{2}\KL(\BP_{j,u_{-j}},\BP_{j,u})}\,,
\end{align*}
which means 
\begin{align*}
    \BE_u[N_{u_j,j}(T)]\leq \BE_{u_{-j}}[N_{u_j,j}(T)]+T\sqrt{\frac{1}{2}\KL(\BP_{j,u_{-j}},\BP_{j,u})}\,.
\end{align*}
Thus,
\begin{align}
        \sum_{u_j\in[n]\setminus u_{-j}} \BE_u[N_{u_j,j}(T)]
        \leq& \sum_{u_j\in[n]\setminus u_{-j}}\BE_{u_{-j}}[N_{u_j,j}(T)]+T\sum_{u_j\in[n]\setminus u_{-j}}\sqrt{\frac{1}{2}\KL(\BP_{j,u_{-j}},\BP_{j,u})}\notag \\
        \leq& \sum_{u_j\in[n]}\BE_{u_{-j}}[N_{u_j,j}(T)]+T\sum_{u_j\in[n]\setminus u_{-j}}\sqrt{\frac{1}{2}\KL(\BP_{j,u_{-j}},\BP_{j,u})}\notag\\
        \leq& T+T\sum_{u_j\in[n]\setminus u_{-j}}\sqrt{\frac{1}{2}\KL(\BP_{j,u_{-j}},\BP_{j,u})}\,.\label{eq:lowerDecom}
\end{align}

Let $Y^t=(Y_1,Y_2,\dots,Y_t)\in\{0,1\}^{m\times t}$ be the sequence of feedback received by the learner up to round $t$ and $y^t=(y_1,y_2,\dots,y_t)\in\{0,1\}^{m\times t}$ be a specific value taken by $Y^t$.
Since the learner is deterministic, $Y^t$ uniquely determines $N_{i,j}(t)$ for $j\in[m]$ and $i\in[n]$.
Specifically,
we have $\BP_{j,u}(\cdot\big|y^T)=\BP_{j,u_{-j}}(\cdot\big|y^T)$. Let $\BP_u^t$ be the law of $Y^t$  when the learner is interacting with the ranking problem determined by $\ell_u$. Then we have 
\begin{align*}
   \BP_{j,u_{-j}}(J_{j,T}=i)=\sum_{y^T}\BP_{j,u_{-j}}(J_{j,T}=i\big|y^T)\BP_{u_{-j}}^T(y^T) \, .
\end{align*}

Thus,
\begin{align*}
    \KL(\BP_{j,u_{-j}},\BP_{j,u})=& \sum_{i\in[n]} \left[\sum_{y^T}\BP_{j,u_{-j}}(J_{j,T}=i\big|y^T)\BP_{u_{-j}}^T(y^T)\right]\log\frac{\sum_{y^T}\BP_{j,u_{-j}}(J_{j,T}=i\big|y^T)\BP_{u_{-j}}^T(y^T)}{\sum_{y^T}\BP_{j,u}(J_{j,T}=i\big|y^T)\BP_{u}^T(y^T)}\\
    \leq& \sum_{i\in[n]}\sum_{y^T}\BP_{j,u_{-j}}(J_{j,T}=i\big|y^T)\BP_{u_{-j}}^T(y^T)\log\frac{\BP_{j,u_{-j}}(J_{j,T}=i\big|y^T)\BP_{u_{-j}}^T(y^T)}{\BP_{j,u}(J_{j,T}=i\big|y^T)\BP_{u}^T(y^T)}\\
    =& \sum_{y^T}\left[\sum_{i\in[n]}\BP_{j,u_{-j}}(J_{j,T}=i\big|y^T)\right]\BP_{u_{-j}}^T(y^T)\log\frac{\BP_{u_{-j}}^T(y^T)}{\BP_{u}^T(y^T)}\\
    =& \KL(\BP_{u_{-j}}^T,\BP_{u}^T)\,,
\end{align*}
where the inequality uses Jensen's inequality and the fact that $f(x_1,x_2)=x_1\log\frac{x_1}{x_2}$ is a convex function whose Hessian matrix $\begin{bmatrix} \frac{1}{x_1} & -\frac{1}{\x_2} \\ -\frac{1}{\x_2} & \frac{x_1}{x_2^2}\end{bmatrix}$ is positive semi-definite.
The second equality comes from $\BP_{j,u}(\cdot\big|y^T)=\BP_{j,u_{-j}}(\cdot\big|y^T)$.

According to the chain rule for KL divergence, we have

\begin{align*}
    & \KL(\BP_{u_{-j}}^T,\BP_{u}^T)\\
    =& \KL(\BP_{u_{-j}}^1,\BP_{u}^1)+\sum_{t=2}^T\sum_{y^{t-1}}\BP_{u_{-j}}^{t-1}(y^{t-1})\KL(\BP_{u_{-j}}^t(\cdot|y^{t-1}),\BP_{u}^t(\cdot|y^{t-1}))\\
    =& \KL(\BP_{u_{-j}}^1,\BP_{u}^1)+\sum_{t=2}^T\left(\sum_{y_{t-1}:i_{t,j}=u_j}\BP_{u_{-j}}^{t-1}(y^{t-1})\KL\left(\frac{1}{2},\frac{1}{2}-\Delta\right)\right.\\
    & \left.+\sum_{y^{t-1}:i_{t,j}\neq u_j}\BP_{u_{-j}}^{t-1}(y^{t-1})\KL\left(\frac{1}{2},\frac{1}{2}\right)\right)\\
    =& \KL\left(\frac{1}{2},\frac{1}{2}-\Delta\right)\BE_{u_{-j}}[N_{u_j,j}(T)]\\
    \leq& 8\Delta^2\BE_{u_{-j}}[N_{u_j,j}(T)]\,.
\end{align*}

Then,
\begin{align*}
    \sum_{u_j\in[n]\setminus u_{-j}}\sqrt{\frac{1}{2}\KL(\BP_{j,u_{-j}},\BP_{j,u})} \leq& \sqrt{\frac{n-m+1}{2}\sum_{u_j\in[n]\setminus u_{-j}}\KL(\BP_{j,u_{-j}},\BP_{j,u})}\\
    \leq& \sqrt{(n-m+1)\sum_{u_j\in[n]\setminus u_{-j}}4\Delta^2\BE_{u_{-j}}[N_{u_j,j}(T)]}\\
    =& 2\Delta\sqrt{(n-m+1)T}\,.
\end{align*}
Recall that we have assumed $n\geq \max\{m+3, 2m\}$. Choosing $\Delta=\frac{1}{8}\sqrt{\frac{n-m+1}{T}}$ and substituting the above into Eq.\eqref{eq:lowerDecom} shows that
\begin{align*}
    \sum_{u_j\in[n]\setminus u_{-j}} \BE_u[N_{u_j,j}(T)]\leq \frac{(n-m+1)T}{2}\,.
\end{align*}


Thus,
\begin{align}
    \sum_{u\in\cS_n(m)}R_u(T)
=& \Delta\sum_{j=1}^m\sum_{u_{-j}\in\cS_{n}(m-1)}\sum_{u_j\in[n]\setminus u_{-j}} (T-\BE_u[N_{u_j,j}(T)])\notag\\
\geq& \frac{n!}{2(n-m+1)!}\Delta m(n-m+1)T\notag\\
=& \frac{n!}{16(n-m)!}m\sqrt{(n-m+1)T}\,.\label{eq:lower10}
\end{align}

Since $|\cS_n(m)|=\frac{n!}{(n-m)!}$ holds, there exists an $u\in\cS_n(m)$ such that 
\begin{align*}
    R_u(T)\geq \frac{1}{16}m\sqrt{(n-m+1)T}\,.
\end{align*}


\end{proof}

\section{Omitted Details for Algorithm \ref{alg:ftrlpbm}}\label{app:opt}

This section discusses the remaining issue in the implementation of Algorithm \ref{alg:ftrlpbm}:  the sampling rule for choosing actions.

\subsection{Decomposition of Subpermutation Matrices}

The sampling of $X_t$ (line 4 of Algorithm \ref{alg:ftrlpbm}) requires finding a distribution over the action set $\cX$ with mean $x_t$. Specifically, we wish to express $x_t$ as a convex combination of a group of $n\times m$ subpermutation matrices. Our method is presented in Algorithm \ref{alg:decomposition}. 
Following \citet{kale2010non}, we first complete $x_t$ into a doubly stochastic matrix $W\in \BR^{n\times n}$ satisfying each element $W_{i,j}\geq 0$ and $\sum^n_{j=1}W_{i,j}=\sum^n_{j=1}W_{j,i}=1$ for any $i\in [n]$ in $O(n^2)$ time (Line \ref{alg:decomposition:l1} in Algorithm \ref{alg:decomposition}). 
Then we apply the Algorithm 1 in \cite{helmbold2009learning} (Line 2-8 in Algorithm \ref{alg:decomposition})
to decompose the doubly stochastic matrix $W$ into its convex combination.
At each iteration, Algorithm \ref{alg:decomposition} finds a permutation matrix $\Pi^k$ such that $W_{i,\Pi^k(i)}>0$ for any $i\in [n]$. 
This could be viewed as the problem of maximal matching in a bipartite graph where edge $e_{i,j}$ has weight $W_{i,j}$ and could be solved by the Hopcroft–Karp algorithm at the cost of $O(n^{2.5})$ time. 
\citet{helmbold2009learning} further show that the doubly stochastic matrix $W$ can be decomposed into the convex combination of at most $n^2-2n+2$ permutation matrices.
Consequently, there are at most $O(n^2)$ iterations on Line \ref{alg3:line2} of Algorithm \ref{alg:decomposition} and Algorithm \ref{alg:decomposition}
 will end in $O(n^{4.5})$ time.

\begin{algorithm}[tbh!]
\caption{Permutation Matrix Decomposition (on Line 4 of Algorithm \ref{alg:ftrlpbm})}
\label{alg:decomposition}
\textbf{Input}: $x_t$
\begin{algorithmic}[1]
\STATE Let $W \in \BR^{n\times n}$ where $\forall i\in [n]$, $W_{i,j}=x_{t,i,j}$ if $1\leq j\leq m$ and $W_{i,j}=\frac{1}{n-m}\left(1-\sum^m_{j=1} x_{t,i,j}\right)$ if $m< j\leq n$. Let $k=0$. \label{alg:decomposition:l1}
\WHILE{$\|W\|_1\neq 0$} \label{alg3:line2}
    \STATE $k:=k+1$.
    \STATE Find a permutation $\Pi^k$ such that $W_{i,\Pi^k(i)}$ is positive for any $i\in [n]$.
    \STATE $\gamma_k := \min_{i\in [n]} W_{i,\Pi^k(i)}$.
    \STATE $W:= W-\gamma_k \Pi^k$.
\ENDWHILE
\STATE Randomly sample and return a permutation $\Pi\in \{\Pi^1,\cdots,\Pi^k \}$ with probabilities $\{\gamma_1,\cdots,\gamma_k \}$.
\end{algorithmic}
\end{algorithm}

\section{Experiments on Real-world Data}\label{app:real_data}
\begin{figure*}[tbh!]
\centering
\includegraphics[width=0.9\textwidth]{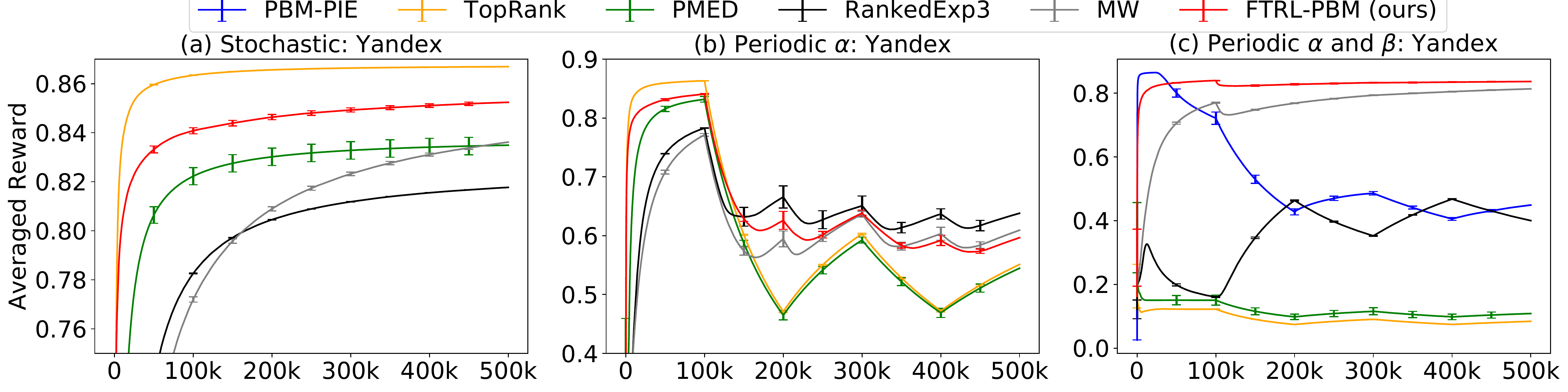}
\caption{This figure shows empirical comparisons between our $\ftrlpbm$ and \texttt{TopRank}, \texttt{PBM-PIE}, \texttt{PMED}, \texttt{RankedExp3} and \texttt{MW} in stochastic and periodic environments constructed by real-world data. 
The metric is the averaged rewards and the error bars are computed in the same way as in Fig.\ref{fig:exp}.
}
\label{fig:exp2}
\end{figure*}
This section presents the experimental results on real-world data. 
The real-world experiment is conducted on the \emph{Yandex} dataset \cite{yandex}. 
There are total $167$m search queries where each query is associated with $10$ items, 
\emph{i.e.} URL links, on $10$ positions. We first take $10$ most frequent items and then take the queries which only consist of these $10$ items. As a result, we get $557,574$ queries. Then we use the EM algorithm \cite{dempster1977maximum} to extract the item attractiveness and position examination probabilities like \cite{lagree2016multiple,komiyama2017position}. 
The resulting item attractiveness are 
$\alpha=(0.894, 0.231, 0.139, 0.0745, 0.0585, 0.0424, 0.0237, 0.0234, 0.0231, 0.0178)$ and the position examination probabilities are $\beta=(0.891, 0.227, 0.0778, 0.0412, 0.0378)$. Like synthetic data, we also construct stochastic and periodic environments. The results are shown in Fig.\ref{fig:exp2}(a)(b)(c). The performances of \texttt{PBM-PIE} are roughly $0.2$, far below others, and fall out of the shown ranges in (a)(b). The results are similar to those in synthetic settings, showing the simultaneous learning abilities of our algorithm $\ftrlpbm$.
\end{document}